\def\year{2020}\relax
\documentclass[letterpaper]{article} 
\usepackage{aaai20}  
\usepackage{times}  
\usepackage{helvet} 
\usepackage{courier}  
\usepackage[hyphens]{url}  
\usepackage{graphicx} 
\usepackage{graphicx}  
\nocopyright
\usepackage[]{algorithm}
\usepackage{algpseudocode}
\usepackage[]{algorithmicx}
\usepackage{amsthm}
\usepackage{amsmath}
\usepackage{amsfonts}

\newtheorem{lemma}{Lemma}
\newtheorem{theorem}{Theorem}
\newtheorem{assumption}{Assumption}
\DeclareMathOperator*{\argmin}{arg\,min}

\title{Weighted Sampling for Combined Model Selection and Hyperparameter Tuning}
\author{Dimitrios Sarigiannis, Thomas Parnell, Haralampos Pozidis\\
IBM Research\\
S{\"a}umerstrasse 4, 8803 R{\"u}schlikon, Switzerland\\
saridimi@gmail.com, tpa@zurich.ibm.com, hap@zurich.ibm.com 
}
 
\begin{document}

\maketitle

\begin{abstract}
The combined algorithm selection and hyperparameter tuning (CASH) problem is characterized by large hierarchical hyperparameter spaces. 
Model-free hyperparameter tuning methods can explore such large spaces efficiently since they are highly parallelizable across multiple machines. 
When no prior knowledge or meta-data exists to boost their performance, these methods commonly sample random configurations following a uniform distribution. 
In this work, we propose a novel sampling distribution as an alternative to uniform sampling and prove theoretically that it has a better chance of finding the best configuration in a worst-case setting.
In order to compare competing methods rigorously in an experimental setting, one must perform statistical hypothesis testing.
We show that there is little-to-no agreement in the automated machine learning literature regarding which methods should be used.
We contrast this disparity with the methods recommended by the broader statistics literature, and identify a suitable approach.
We then select three popular model-free solutions to CASH and evaluate their performance, with uniform sampling as well as the proposed sampling scheme, across 67 datasets from the OpenML platform.
We investigate the trade-off between exploration and exploitation across the three algorithms, and verify empirically that the proposed sampling distribution improves performance in all cases.
\end{abstract}

\section{Introduction}
In recent years, it has become apparent that the proliferation of skills required to develop end-to-end machine learning solutions has not kept pace with the rapid growth in expectations surrounding the potential of such techniques. 
This imbalance has led to the desire for automated machine learning (AutoML) systems.
The goal of an AutoML system is to automate the entire process of building a machine learning model.
Typically, this includes data cleaning, data pre-processing, feature engineering, model selection, hyperparameter tuning and in some cases even ensemble-building. 
In recent years, software frameworks such as Auto-WEKA \cite{Thornton:2013:ACS:2487575.2487629}, auto-sklearn \cite{NIPS2015_5872}, TPOT \cite{pmlr-v64-olson_tpot_2016} and H2O Driverless AI \cite{driverless} have appeared that offer exactly this functionality. 

\paragraph{The CASH Problem}
At the heart of any automated machine learning framework is an optimization problem over a large hierarchical parameter space. 
A machine learning solution is typically structured as a pipeline comprising multiple components.
Each component has a different set of hyperparameters that must be tuned to achieve best-in-class accuracy. 
In this work, we focus on a restricted set of such pipelines, considering solely the problem of \textbf{C}ombined \textbf{A}lgorithm \textbf{S}election and \textbf{H}yperparameter tuning (henceforth referred to as CASH). Specifically, this is a joint optimization problem involving selecting which machine learning model to use (e.g. random forest vs. gradient-boosted decision trees) and how to tune the model hyperparameters. The hierarchical configuration space associated with the CASH problem is illustrated in Figure \ref{fig:configspace}. Mathematically, the problem of finding the optimal model $\lambda^*$ and its corresponding hyperparameter configuration $\alpha^*$ is defined:
\begin{equation}
	\left(\lambda^*, \alpha^*\right) = \argmin_{\lambda\in\{1,\ldots,M\}, \alpha\in\mathcal{A}(\lambda)} \mathcal{L}_{valid}(\lambda, \alpha),\label{eq:cash}
\end{equation} 
where $M$ is the total number of models, $\mathcal{A}(\lambda)$ is the set of all possible hyperparameter configurations for model $\lambda$ and $\mathcal{L}_{valid}$ is a loss function evaluated over a validation dataset. Crucially, it is important that the solution found also generalizes to unseen data, which is typically verified by evaluating the same loss function over a test set unseen to the optimization routine.

\paragraph{Model-based vs. Model-free optimization.}
Solutions for solving equation (\ref{eq:cash}) fall roughly into two categories: model-based and model-free. 
Model-based approaches are typically based on Bayesian optimization \cite{Hutter:2011:SMO:2177360.2177404,NIPS2011_4443}: models and their hyperparameter configurations (henceforth jointly referred simply as a configuration) are evaluated sequentially, whereby the next configuration to try is determined by maximizing some acquisition function involving a surrogate model (often based on Gaussian processes). 
While model-based approaches often lead to superior accuracy, such methods are inherently difficult to parallelize and maximization of the acquisition function and/or updating the surrogate model can often be very slow. 
On the other hand, model-free approaches \cite{Bergstra:2012:RSH:2188385.2188395,jamieson2014best,pmlr-v51-jamieson16,DBLP:journals/corr/LiJDRT16} simply involve randomly sampling configurations and evaluating them, possibly under what is known as a resource constraint. 
These methods are inherently easy to parallelize and distribute across a cluster of machines, and make no assumptions on the structure of the underlying space. 
For this reason, in this paper we focus on model-free methods.
In particular, we focus on how the sampling distribution affects the accuracy of the resulting solution.

\paragraph{Comparing AutoML methods}
A major challenge in the AutoML field is how different schemes should be compared. There may exist datasets for which the differences between competing methods are relatively small, and others where differences are very big. It is crucial to have a consistent way to measure the differences between methods and be able to perform rigorous statistical hypothesis testing on those measurements, to ensure that differences observed are truly significant. Despite the importance of this challenge, there is little discussion in the AutoML literature on this topic and a wide disparity in the evaluation methodology and the statistical techniques applied. Addressing this problem is another focus of this paper. 

\paragraph{Contributions}
The contributions of this paper are as follows:
\begin{itemize}
	\item We propose a novel sampling distribution as an alternative to uniform sampling and prove theoretically that it has a better chance of finding the best configuration in a worst-case setting.
	\item We select three popular model-free hyperparameter tuning algorithms and perform a large empirical study, using 67 datasets from OpenML \cite{OpenML2013}, with uniform sampling as well as the proposed scheme.
	\item We review the state-of-the-art in statistical hypothesis testing in the context of machine learning, and propose a systematic approach for comparing multiple AutoML methods across a collection of datasets.
	\item We investigate the trade-off between exploration and exploitation across the three algorithms and verify empirically that the proposed sampling distribution improves average rank performance in all cases.
\end{itemize}

\begin{figure}[!htb]
    \center{\includegraphics[width=\columnwidth]{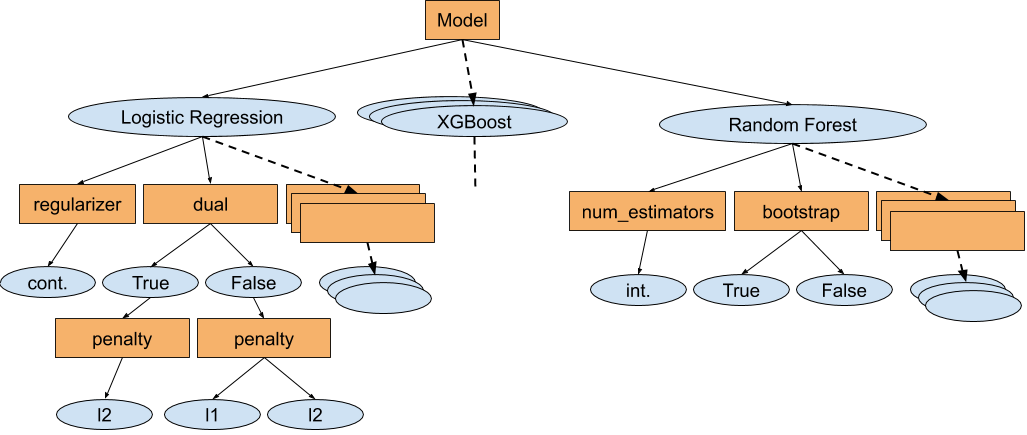}}
    \caption{\label{fig:configspace} Hierarchical configuration space of the models and their corresponding hyperparameters.}
\end{figure}

\section{Model-free Optimization for CASH} 
\label{methods}

In this section we describe three popular model-free methods that can be applied to solve the CASH problem.

\subsection{Random Search}

The simplest and most well-known model-free approach for solving (\ref{eq:cash}) is random search (RS) \cite{Bergstra:2012:RSH:2188385.2188395}. The method is very simple: one draws $n$ different configurations (models and their hyperparameters) at random according to some sampling distribution. Each configuration is then trained on the full training set, and the configuration that results in the minimal validation loss is selected as the winner. Since all evaluations are independent, RS is embarrassingly parallel. Despite its simplicity, it is widely acknowledged that RS is a competitive baseline for AutoML optimization \cite{li2019random}.

\subsection{Successive Halving}

The next model-free approach we consider is successive halving (SH) \cite{jamieson2014best,pmlr-v51-jamieson16}. 
Crucial to this method is the notion of a \textit{resource}. 
Namely, some quantity that if lowered, will reduce the training time, and if raised will increase the training time. 
For iterative learning algorithms (e.g. stochastic gradient descent), an appropriate resource would be the number of iterations. 
When considering the CASH problem, where hyperparameters across different models are not consistent (e.g. learning algorithms may not all be iterative), we must define resource in some alternative manner.
Following the approach of \cite{DBLP:journals/corr/LiJDRT16} we define the resource to be the size of a stratified subsample of the training dataset, thus overcoming the issue of model heterogeneity.

With this definition in hand, SH begins by randomly sampling $n_0$ configurations according to some sampling distribution and evaluates them using a minimum resource $r_{min}$ (for instance $r_{min}=0.1$ implies $10\%$ of the training examples). 
The algorithm then identifies the best $n_0\eta^{-1}$ configurations and carries them over into the next \textit{rung} where said configurations are evaluated using a resource of $r_{min}\eta$.
The parameter $\eta$ is a hyperparameter of the SH method that controls how aggressively configurations are eliminated. 
The method continues as above, reducing the number of configurations by a power of $\eta$, and increasing the resource by a power of $\eta$ until the maximal resource is attained, which in the CASH context corresponds to a resource of one (i.e., all of the training examples).
In total, the number of rungs is given by $1 + s_{max}$ where $s_{max} = \lfloor -\log_{\eta}(r_{min})\rfloor$.
In order to ensure there is at least one configuration evaluated with the maximal resource we also require that $n_0 \geq \eta^{s_{max}}$.

The steps of the SH are provided in full in Algorithm \ref{SH_Algorithm}.
It should be noted that within each rung, all evaluations can be executed in parallel, whereas between rungs there exist dependencies. 
However, there exists some recent work that tries to effectively overcome these constraints to achieve more efficient parallel and distributed implementations of SH \cite{DBLP:journals/corr/abs-1810-05934}.

\begin{algorithm}[h]
\caption{Successive Halving}
\label{SH_Algorithm}
\begin{algorithmic}[1]
\Require{initial number of configurations $n_0$, minimum resource $r_{min}$, scaling factor $\eta$, sampling distribution $p(\lambda,\alpha)$}
\State $s_{max} \gets \lfloor -\log_{\eta}(r_{min})\rfloor$
\Ensure{$n_0 >= \eta^{s_{max}}$}
\Statex
\State $T \gets sample\_configurations(n, p(\lambda,\alpha))$
\For{$i \in \{0,1,...,s_{max}\}$}                    
\State $n_i \gets \lfloor n \eta^{-i}\rfloor$
\State $r_i \gets \eta^{-s_{max}+i}$
\State $L \gets eval\_and\_return\_val\_loss(\theta, r_i):\theta \in T$
\State $T \gets top\_k(T,L,n_i/\eta)$
\EndFor
\Statex
\State \Return {Configuration with the smallest intermediate loss seen so far in T}
\end{algorithmic}
\end{algorithm}

In order to compare SH against RS and other methods, we need to define the notion of a \textit{budget}. 
The budget is defined to be the total amount of resource used (i.e., effective number of times the full training data is processed), and is given mathematically as:
\begin{equation}
b_{SH} = \sum_{i=0}^{s_{max}} n_i r_i,
\label{eqn:sh}
\end{equation}
where $n_i$ and $r_i$ are as defined in Algorithm \ref{SH_Algorithm}. 
By comparison, the budget for RS is simply the number of configurations evaluated.

\paragraph{Exploitation vs. Exploration} Let us assume that $\eta=3$ is fixed, and observe that by selecting different values for the minimum resource $r_{min}$ and the initial number of configurations $n_0$, it is possible to obtain instances of SH with an equivalent budget, but with significantly different elimination schedules. 
In the corner-case, if we set $r_{min}=1.0$ and $n_0=n$, we obtain a schedule ($SH_0$) that is equivalent to RS with budget $n$. 
Because this schedule only evaluates configurations using the maximal resource, we refer to it as the most \textit{exploitative} schedule.
On the other hand, if we set $r_{min}=1/9$ and $n_0=3n$, we obtain a different schedule ($SH_2$) which has the same budget: 
\begin{equation}
	B_{SH_2} = 3n\times\frac{1}{9} + n\times\frac{1}{3} + \frac{n}{3}\times 1 = n. 
\end{equation}
Since this schedule is able to evaluate many more configurations in the initial rung (albeit with reduced resource), we refer to this schedule as a more \textit{explorative} schedule. 
Three different such schedules are illustrated in Table \ref{tab:brackets}.

\begin{table}[ht]
\begin{center}
\begin{tabular}{cccccc}
    \hline
    \multicolumn{2}{c}{$SH_2$} & \multicolumn{2}{c}{$SH_1$} & \multicolumn{2}{c}{$SH_0$}\\
    $n_i$&$r_i$ & $n_i$&$r_i$ & $n_i$&$r_i$ \\
    $3n$ & $1/9$ & $3n/2$ & $1/3$ & $n$ & $1$ \\
    $n$ & $1/3$ & $n/2$ & $1$ & -&- \\
    $n/3$  & $1$  & -&- & -&- \\
    \hline
\end{tabular}
\end{center}
\caption{Three SH schedules with equivalent budget ranging from the most explorative ($SH_2$) to the most exploitative ($SH_0$).}
\label{tab:brackets}
\end{table}

\subsection{Hyperband}
One of the main issues with SH is how a practitioner should decide which of the aforementioned schedules to use. There may be some optimization problems where a more exploitative schedule performs better, and others where a more explorative schedule is desirable.
This is exactly the problem considered in \cite{DBLP:journals/corr/LiJDRT16}, in which the authors propose a new method, Hyperband, that tries to seamlessly handle the exploration-exploitation trade-off.
The main concept is very simple: execute a number of SH schedules in parallel (with varying values for $r_{min}$) and output the best configuration found by any of them.
Each instance of SH is referred to as a \textit{bracket}.
To give a concrete example, a Hyperband instance would execute the three brackets from Table \ref{tab:brackets} ($SH_0$, $SH_1$ and $SH_2$) in parallel. 
The budget of such a scheme is given by:
\begin{equation}
b_{HB} = \sum_{i=0}^{2} b_{SH_i} = 3n.
\end{equation}
Therefore, in order to fairly compare this instance of Hyperband with RS, the RS should sample $3n$ configurations and evaluate them on the full resource.

\section{Weighted Sampling of Models} 
\label{theory}

\subsection{Main Idea}
The model-free approaches for solving CASH discussed in the previous section generally involve sampling models (and their hyperparameters) according to a uniform distribution. 
While there have been some efforts to alter this distribution via adaptive Bayesian linear regression models \cite{valkovsimple}, such methods perform best when a large amount of meta-data is available.
Using meta-data to improve SH was also considered in \cite{sommer2019}, in which the sampling distribution remains uniform, but the decision regarding which configurations to eliminate is taken using a meta-model trained on existing tuning experiments.
In such an approach, the evaluations at earlier rungs are akin to so-called landmarking meta-features \cite{Pfahringer:2000:MLV:645529.658105}.
While this line of research is promising, in practice such meta-data may not exist and may be costly to obtain.
Instead, in this paper we propose a non-uniform sampling distribution based on a very simple heuristic: as the number of model hyperparameters increases linearly, one requires exponentially more budget (or model evaluations) in order to identify the optimal configuration. Mathematically, we propose to sample model $\lambda$ according to the following probability:
\begin{equation}
	p_\lambda = \frac{2^{N_\lambda}}{\sum_{\lambda'=1}^M 2^{N_{\lambda'}}},\label{eq:heuristic}
\end{equation}
where $N_\lambda\in\mathbb{Z}_+$ is the number of hyperparameters of model $\lambda$.
In the rest of this section, we will provide a theoretical motivation for such a sampling distribution.

\subsection{Theoretical Motivation}

In this section we prove that, under a worst-case scenario for CASH, there exists a weighted sampling scheme that consistently outperforms the uniform sampling scheme. The worst-case scenario that we consider is that for any new dataset, the optimal model is drawn uniformly at random, and the optimal configuration of that model is similarly drawn uniformly at random. We assume that each model has an integer number of continuous hyperparameters. We formalize this scenario with the following assumptions:
\begin{assumption} \label{a1} The optimal model $\lambda^*_i$ is a random variable distributed according to a discrete uniform distribution:
	\begin{equation*}
		\lambda^*_i \sim U\{1,M\},
	\end{equation*}
	where $M$ is the total number of models. 
\end{assumption}

\begin{assumption} \label{a2} Given model $\lambda\in\{1,\ldots,M\}$, the optimal configuration of the $n$-th hyperparameter of the model is a random variable distributed accordingly to a continuous uniform distribution. That is, for $n=1,2,\ldots,N_\lambda$:
	\begin{equation*}
	\alpha(\lambda)^*_{i,n} \sim U(l\left(\lambda)_n, u(\lambda)_n\right),
	\end{equation*}
where $l(\lambda)_n\leq u(\lambda)_n\in\mathbb{R}$ are the lower and upper limits of the $n$-th hyperparameter of model $\lambda$ respectively.
\end{assumption}

We will now prove the following lemma regarding the performance of RS under a given model-sampling scheme. We assume that the probability of the RS sampling model $\lambda$ is given by $p_\lambda$ and all hyperparameters of all models are sampled uniformly from their ranges.
\begin{lemma}\label{lemma:prob_rs}
	Under the assumptions above, the probability of a RS with budget $K$ not finding the optimal model $\lambda^*_i$ and configuration $\alpha^*_{i,1},\ldots,\alpha^*_{i,N_{\lambda^*_i}}$ is given by:
	\begin{equation}
		P_F = \frac{1}{M}\sum_{\lambda=1}^M\left(1 - \frac{p_{\lambda}}{\theta_\lambda}\right)^K,
	\end{equation}
	where $\theta_\lambda$ is given by:
	\begin{equation}
		\theta_\lambda = \prod_{n=1}^{N_\lambda}\left(u(\lambda)_n - l(\lambda)_n\right)
	\end{equation}
\end{lemma}
\begin{proof}
	The probability of failure for a RS with budget $K$ can be expressed:
	\begin{equation}
		P_F = \Pr\left(\bigcap_{k=1}^K\left[\lambda^{(k)}_i = \lambda^*_i, \bigcap_{n=1}^{N_{\lambda^*_i}} \alpha^{(k)}_{i,n} = \alpha^*_{i,n} \right]^c\right)\nonumber,
	\end{equation}
	where $\lambda^{(k)}_i$ is the $k$-th model sampled by the RS for dataset $i$ and $\alpha^{(k)}_{i,n}$ is the $k$-th sampled value of the $n$-th hyperparameter. Note that explicit dependence of the hyperparameter random variable on the selected model has been dropped from notation to ease readability.
	By applying total and conditional laws of probability as well as Assumption (\ref{a1}) and Assumption (\ref{a2}) we obtain:
	\begin{align*}
		&P_F = \frac{1}{M}\sum_{\lambda=1}^M \prod_{n=1}^{N_\lambda}\int_{l(\lambda)_n}^{u(\lambda)_n} d\alpha_n \frac{1}{u(\lambda)_n - l(\lambda)_n} \prod_{k=1}^K  \\
		&\Pr\left( \left[\lambda^{(k)}_i = \lambda, \bigcap_{n=1}^{N_{\lambda}} \alpha^{(k)}_{i,n} = \alpha_n \right]^c \mid \lambda_i^*=\lambda, \bigcap_{n=1}^{N_{\lambda}} \alpha^*_{i,n} = \alpha_n \right)
	\end{align*}
	Now, by introducing the sampling distribution on the model and assuming the hyperparameters of each model are sampled uniformly within their ranges, we obtain:
	\begin{align*}
	P_F &= \frac{1}{M}\sum_{\lambda=1}^M \left(1 - p_\lambda\prod_{n=1}^{N_{\lambda}}\frac{1}{u(\lambda)_n-l(\lambda)_n}\right)^K \\ 	\times    &\prod_{n=1}^{N_\lambda}\int_{l(\lambda)_n}^{u(\lambda)_n} d\alpha_n \frac{1}{u(\lambda)_n - l(\lambda)_n} \\
	&= \frac{1}{M}\sum_{\lambda=1}^M \left(1 - \frac{p_\lambda}{\theta_\lambda}\right)^K \qedhere
	\end{align*}
\end{proof}

Applying Lemma \ref{lemma:prob_rs} the above to the uniform sampling case where $p_\lambda = 1/M$ we find that the probability of failure can be expressed:
\begin{equation}\label{eq:pf_u}
	P_F^{(U)} = \frac{1}{M}\sum_{\lambda=1}^M \left(1-\frac{1}{M\theta_\lambda}\right)^K
\end{equation}
Now, consider a RS with a non-uniform sampling probability is given by:
\begin{equation}
	p_\lambda = \frac{\theta_\lambda}{\sum_{\lambda'=1}^M\theta_{\lambda'}}\label{eq:theta_samp}
\end{equation}
the probability of failure can be expressed:
\begin{equation}\label{eq:pf_nu}
	P_F^{(W)} = \left(1-\frac{1}{\sum_{\lambda=1}^M \theta_\lambda}\right)^K
\end{equation}
\begin{theorem}
	Unless $\theta_1=\theta_2=\ldots=\theta_M$, the probability of failure of the RS with weighted sampling is strictly less than that of the RS with uniform sampling.
\end{theorem}
\begin{proof}
	By equation (\ref{eq:pf_nu}), the probability of failure of the weighted scheme is given by:
	\begin{equation*}
	P_F^{(W)} = \left(1-\frac{1}{\sum_{\lambda=1}^M \theta_\lambda}\right)^K = \left(1-\frac{1}{M\sum_{\lambda=1}^M \frac{1}{M} \theta_\lambda}\right)^K \\
	\end{equation*}
	An application of Jensen's inequality reveals:
	\begin{equation*}
	P_F^{(W)} \leq \frac{1}{M}\sum_{\lambda=1}^M \left(1-\frac{1}{M\theta_\lambda}\right)^K  = P_F^{(U)},
	\end{equation*}
	with equality if and only if $\theta_1=\theta_2=\ldots=\theta_M$.
\end{proof}

\paragraph{Relationship to proposed sampling distribution (\ref{eq:heuristic})} The proposed weighted sampling distribution given in (\ref{eq:heuristic}) can be interpreted as a distribution of the form (\ref{eq:theta_samp}) where $u(\lambda)_n-l(\lambda)_n=2$ for all models $\lambda=1,\ldots,M$ and all hyperparameters $n=1\ldots,N_\lambda$. While this approximation does not take into account the hyperparameter ranges, these are unknown in practice and we will show in Section \ref{sec:experiments} that this heuristic leads to a significantly better solution than uniform model sampling, also in the case when models have a mix of continuous and categorical hyperparameters.  

\section{Statistical Comparison of AutoML Methods}

In this section, we will review what statistical techniques are being used in the AutoML literature today, discuss how they relate to the broader literature on statistical comparisons of ML methods, and finally provide a clear recommendation of how multiple AutoML methods should be compared across a collection of datasets.

\subsection{Review of statistical analysis in the AutoML literature}

Statistical analysis presented in the AutoML literature can generally be separated into two distinct categories. There are papers which compare competing methods on a per-dataset basis (i.e., asserting that method A is statistically different to method B on each individual dataset) and those that compare competing methods across a collection of datasets (i.e., verifying that method A is statistically different to method B across the entire collection). 

In terms of the per-dataset approach, there is little consensus on how this comparison should be performed. In \cite{Thornton:2013:ACS:2487575.2487629}, normal evaluation metrics are simply presented in a table with no statistical hypothesis testing to validate if differences are significant. 
Whereas in \cite{DBLP:journals/corr/abs-1807-01774} the average regret is plotted for each method along with the standard deviation measured over multiple runs.
Formal statistical significance testing is performed in \cite{NIPS2015_5872} and \cite{Feurer:2015:IBH:2887007.2887164} using a bootstrap test and a two-sided t-test respectively, although it is not clear whether the p-values have been adjusted to account for the many hypotheses tested.
Finally, in \cite{pmlr-v64-olson_tpot_2016} the authors use the Mann-Witney $U$ test to determine statistically significant wins/losses on each dataset and apply the Bonferroni correction to account for multiple hypotheses. 

Similarly, when AutoML methods are compared across a collection of datasets, there is again little agreement. In \cite{DBLP:journals/corr/abs-1808-06492} metrics such as F1-score or MSE are simply averaged over the collection of datasets, which is problematic since the MSE of a \textit{difficult} dataset may not be directly comparable with the MSE of an \textit{easier} one. In \cite{DBLP:journals/corr/LiJDRT16} and \cite{Feurer:2015:IBH:2887007.2887164} the average ranks of the competing methods are compared, where the average is computed across a collection of datasets, but no statistical analysis is performed to determine whether the differences in rank are indeed significant or not. Finally, in \cite{DBLP:journals/corr/abs-1808-03233}, average ranks are presented along with a standard deviation.

\subsection{What does the broader ML and statistics literature recommend?}

In the widely-cited paper of \cite{Demsar2006} it is argued that it is preferable to compare machine learning methods statistically across a collection of datasets, rather than on a per-dataset basis. The argument provided is that by performing statistical significance testing on a per-dataset basis and counting the number of \textit{significant wins}, one is implicitly assuming that each test can distinguish between a random and a non-random difference. This is not the case, the test can only state the improbability of the observed event assuming that the null hypothesis was correct. Furthermore, in order to apply methods like the Student t-test on a single dataset, one must somehow generate repeated measurements (e.g. by resampling the training/test set) which often violates the underlying assumptions required to apply the test such as normality and/or independence \cite{Dietterich1998}.

When comparing two ML methods across multiple datasets, the approach recommended by \cite{Demsar2006} is to apply the Wilcoxon signed-ranks test. This non-parametric test makes fewer assumptions relative to a Student t-test, and is able to take into account the relative magnitude of differences (as opposed to the simpler sign test). When comparing more than two ML methods one should first apply a family-wise hypothesis test (e.g. the Friedman test with Iman and Davemport extension), and once it has been determined that differences exist within the family, one should proceed to apply pairwise post-hoc testing, 
In this context it is common practice to use the mean-ranks test \cite{nemenyi1962distribution} as a post-hoc test. 
However, it was recently demonstrated in \cite{benavoli2016} that this approach leads to results that depend strongly on the number of methods included in the pool and that by adding or removing methods one can arrive at contradictory conclusions. 
Instead, the authors recommend to use the Wilcoxon signed-ranks test, as recommended when comparing only two methods, while applying appropriate correction techniques to account for multiple hypotheses.

A rigorous evaluation of different correction techniques was further provided in \cite{Garcia2008} and \cite{garcia2010}. 
The most simple such technique is the Bonferroni correction, although it has relatively lower power.
Conversely, the Hommel and the Rom procedure are considered the most powerful, at the expense of significant computational complexity.
The Finner correction is also powerful but is vastly simpler, and is thus recommended in most cases.

\subsection{Proposed method}\label{sec:stats_proc}

Based on the above literature review, we propose the following guidelines for comparing $K>2$ AutoML methods across a collection of datasets:
\begin{enumerate}
	\item Apply the omnibus test (Friedman test with Iman and Davemport extension) to determine whether at least one method performs differently to the others.
	\item Construct a $K\times K$ matrix of raw p-values arising from all-to-all pairwise comparisons using the Wilcoxon signed-rank test across datasets.
	\item Apply the Finner correction to the above matrix to account for multiple hypotheses.
\end{enumerate}
All of the above methods are implemented in the R package \textsc{scmamp} \cite{calvo2016scmamp}, which we will make extensive use of in the following section.

\section{Experimental Results}
\label{sec:experiments}

In this section we will compare different model-free solutions to the CASH problem, with and without weighted sampling, across a collection of 67 datasets.

\begin{table}[ht]
\begin{center}
\begin{tabular}{ ccc }
 \hline
  & num ex & num ft \\
 \hline
 count & 67     & 67 \\
 mean & 1965.05 & 24.90 \\
 std & 1212.33  & 19.35 \\
 min & 1000     & 1 \\
 25\% & 1000.00 & 8.50 \\
 50\% & 1563.00 & 20.00 \\
 75\% & 2330.50 & 38.50 \\
 max & 5000     & 76 \\
 \hline
\end{tabular}
\caption{Summary of OpenML Datasets under study.}
\label{tab:datasets}
\end{center}
\end{table}

\subsection{Experimental Design}

All datasets were obtained from the OpenML platform \cite{OpenML2013} and their characteristics are summarized in Table \ref{tab:datasets}.
A complete list of OpenML dataset IDs is provided in Appendix \ref{app:ids}, and the pre-processing scheme used is provided in Appendix \ref{app:pp}.
All datasets correspond to binary classification problems.
In terms of model selection, we consider a pool of 11 different models. 
The models and their hyperparameters are summarized in Table \ref{tab:models}.
Since the datasets are relatively small, the performance of each optimization method is evaluated using many repetitions in a nested manner.
Firstly, we create a stratified train/test split of each dataset.
We then perform 10 different stratified splits of the training set to create a collection of 10 different train/validation sets.
Internally in each optimization framework, every configuration is trained on each of the 10 training sets, and the validation loss is evaluated on the corresponding validation sets.
The validation loss used to compare different configurations is taken to be the average over all train/validation splits.
Once the best configuration has been identified, we then re-train it using the full training set, and evaluate the loss function on the test set.
The above process is repeated 10 times (with different train/test splits) and the average generalization performance is reported.
In all experiments the logistic loss is used as a loss function.

\begin{table}[ht]
\begin{center}
\resizebox{0.95\columnwidth}{!}{
\begin{tabular}{ lcccc }
 \hline
 Name & \# hp & \# cat & \# int & \# cont \\
 \hline
 RandomForestClassifier & 8 & 3 & 4 & 1 \\
 LogisticRegression & 6 & 4 & 0 & 2 \\
 XGBoost & 11 & 2 & 3 & 6 \\
 GradientBoostingClassifier & 10 & 3 & 4 & 3 \\ 
 AdaBoostClassifier & 2 & 0 & 1 & 1 \\
 BernoulliNB & 3 & 1 & 1 & 1 \\
 GaussianNB & 1 & 0 & 0 & 1 \\
 ExtraTreesClassifier & 8 & 4 & 3 & 1 \\
 KNeighborsClassifier & 3 & 2 & 1 & 0 \\
 LinearDiscriminantAnalysis & 4 & 1 & 1 & 2 \\
 QuadraticDiscriminantAnalysis & 1 & 0 & 0 & 1 \\
 \hline
\end{tabular}}
\caption{Classification Models. For XGBoost we have used the xgboost v0.82 library and for the rest of the classifiers we have used scikit-learn v0.21.2.}
\label{tab:models}
\end{center}
\end{table}

\subsection{Exploration vs Exploitation} \label{exploration_vs_exploitation}

In our first comparison, we compare the three different SH schedules defined in Table \ref{tab:brackets} with a budget of $33$, so that in the most explorative schedule $n_0=99$ configurations are evaluated in the first rung. 
For each schedule, we evaluate SH with uniform model sampling and also with the weighted model sampling defined in equation (\ref{eq:heuristic}). 
The hyperparameters for each model are sampled uniformly from a fixed range in both cases (possibly with some logarithmic transformations). 
We are therefore comparing 6 different schemes across a collection of 67 datasets. 

We will follow the statistical approach defined in Section \ref{sec:stats_proc} to compare these schemes for both the validation loss and the generalization loss. 
Firstly, we perform the omnibus test and find that for both the validation and the generalization results, the p-value is very small ($<2.2e^{-16}$), indicating that differences do indeed exist between the family of 6 schemes.
Next, we compute the matrix of p-values for all pairwise comparisons using the Wilcoxon signed rank test and apply the Finner correction to account for the multiple hypotheses tested.
The matrices of corrected p-values are presented in full in Appendix \ref{app:eval}, where it can be seen that all p-values are less than a threshold of $0.05$, indicating that the null hypothesis can be rejected for all pairwise comparisons. 

In terms of the relative performance, the average ranks for the 6 schemes are displayed in Figure \ref{fig:res_1x_val} and Figure \ref{fig:res_1x_test} for the validation loss and generalization loss respectively. The conclusions we can draw are as follows:
\begin{enumerate}
	\item The relative order of all 6 schemes is consistent across validation and generalization loss.
	\item The more explorative schedules of SH consistently out-perform the more exploitative schedules.
	\item The weighted model sampling (denoted by SH\{0,1,2\}.W) improves the average rank of all three schedules.
\end{enumerate}

\begin{figure}[!htb]
	\center{\includegraphics[width=\columnwidth]{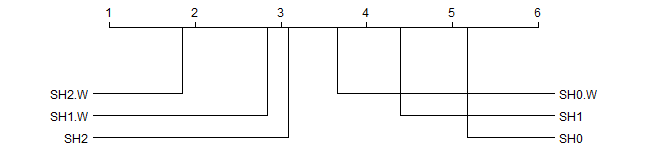}}
	\caption{Average rank in validation loss for the three different schedules of SH (lower is better).}
	\label{fig:res_1x_val} 
\end{figure}
\begin{figure}[!htb]
	\center{\includegraphics[width=\columnwidth]{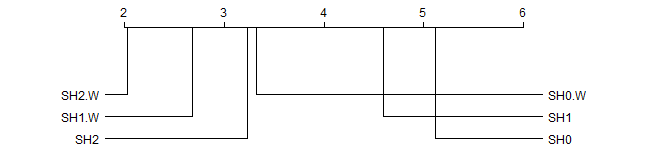}}
	\caption{Average rank in generalization loss for the three different schedules of SH (lower is better).}
	\label{fig:res_1x_test}
\end{figure}

\subsection{Hyperband Evaluation} \label{hyperband_evaluation}

In this section we would like to assess how effectively Hyperband can automate the choice between exploration and exploitation.
Specifically, we will evaluate an instance of Hyperband consisting of the three brackets defined in Table \ref{tab:brackets}, again with the parameter $n=33$, thus having a budget of $99$ in total.
For comparison, we will compare with RS and the most explorative SH schedule, both with an equivalent budget (e.g. RS samples $n=99$ configurations).
Based on the results of the previous section, we have a strong indication that explorative schedules are well suited to the CASH problem. 
The question we would like to ask is the following: can Hyperband, without this knowledge, perform similarly well to an explorative schedule of SH with an equivalent budget?
We will also evaluate each of the three methods (Hyperband, RS and SH) with and without weighted model sampling.

\begin{figure}[!htb]
	\center{\includegraphics[width=\columnwidth]{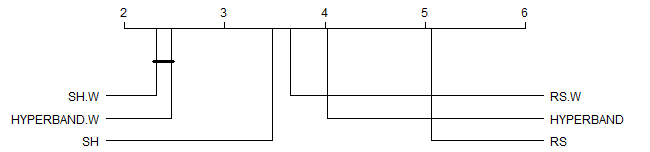}}
	\caption{Average rank in validation loss for the Hyperband, SH and RS (lower is better).}
	\label{fig:res_3x_val} 
\end{figure}

\begin{figure}[!htb]
	\center{\includegraphics[width=\columnwidth]{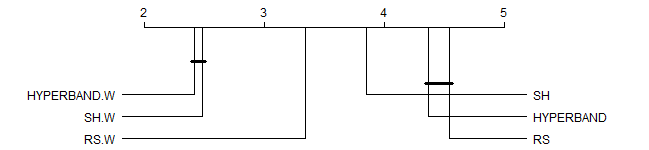}}
	\caption{Average rank in generalization loss for the Hyperband, SH and RS (lower is better).}
	\label{fig:res_3x_test} 
\end{figure}

As before, we firstly perform the omnibus test and find that for both the validation and the generalization results, the p-value is very small ($<2.2e^{-16}$), again indicating that differences do indeed exist between the family of 6 schemes. The matrices of corrected p-values for all of the pairwise comparisons can also be found in Appendix \ref{app:eval}. The results are summarized in Figure \ref{fig:res_3x_val} and Figure \ref{fig:res_3x_test} for the validation loss and generalization loss respectively. In these figures we again plot the average rank, and denote with horizontal bars the schemes for which the null hypothesis could not be rejected (i.e., the p-value for the pairwise comparison was $>0.05$).
 The conclusions we can draw are as follows:
\begin{enumerate}
	\item Again, the relative order of all 6 schemes is consistent across validation and generalization loss.
	\item Without weighted model sampling, explorative SH outperforms both RS and Hyperband, and in the generalization loss the differences between Hyperband and RS are not statistically significant (somewhat confirming the results of \cite{DBLP:journals/corr/LiJDRT16}).
	\item Using weighted model sampling, the performance of Hyperband significantly improves, and in fact the differences between Hyperband and SH are not significantly different either in validation loss or generalization loss.
\end{enumerate}

We have observed that with weighted model sampling, Hyperband is very effective: the practitioner that uses Hyperband, without the knowledge that $SH2$ is likely to work well, is likely to obtain similar results to the practitioner that tries to explicitly optimize the explore/exploit trade-off. However, the same assertion cannot be made if one was to apply Hyperband with uniform model sampling. Why might this be the case? 
In Figure \ref{fig:brackets}, we show the total number of times the different brackets of Hyperband produce winning configurations, for all of the datasets and all of the train/test repetitions, with and without weighted sampling. As expected, we see that the most explorative brackets win most of the time. However, when using weighted model sampling, we observe that the distribution becomes more \textit{equalized}. The ideal situation for applying Hyperband would be when such a distribution is uniform (i.e., one never knows whether explorative vs. exploitative brackets are preferable), thus explaining why the performance of Hyperband improves significantly (relative to explorative SH of an equivalent budget) when using weighted sampling. 

\begin{figure}[!htb]
	\center{\includegraphics[width=\columnwidth]{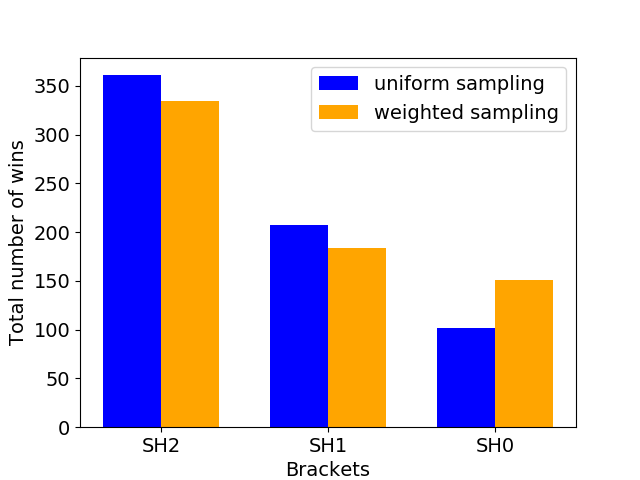}}
	\caption{Total number of times each bracket outputs the winning configuration (for all 67 datasets and 10 train/test splits).}
	\label{fig:brackets} 
\end{figure}

\section{Conclusion}

We propose a weighted sampling distribution for model-free optimization of the combined algorithm selection and hyperparameter tuning problem.
We prove theoretically, under worst-case assumptions, that this distribution out-performs uniform sampling.
Additionally, we recommend a robust procedure for statistical comparison of competing AutoML frameworks over a collection of datasets.
We then evaluate the performance of RS, SH and Hyperband with and without weighted model sampling.
Our findings are threefold: (a) weighted sampling improves performance of all three schemes, (b) explorative SH schedules tend to out-perform exploitative schedules, and (c) that weighted sampling effectively enables Hyperband to successfully automate the explore/exploit trade-off.

\appendix

\section{OpenML IDs} 
\label{app:ids}
In this work we used datasets with the following OpenML IDs:
3, 31, 44, 715, 720, 723, 728, 737, 740, 741, 743, 751, 772, 797, 799, 806, 813, 837, 845, 849, 866, 871, 903, 904, 910, 912, 913, 914, 917, 934, 953, 958, 962, 971, 979, 983, 991, 995, 1020, 1049, 1050, 1067, 1068, 1444, 1453, 1462, 1487, 1494, 1504, 1547, 1558, 40646, 40647, 40648, 40649, 40650, 40680, 40701, 40702, 40704, 40706, 40713, 40983, 40999, 41005, 41007, 41156.

\section{Dataset Preprocessing} 
\label{app:pp}
We applied the sklearn.impute.SimpleImputer to fill any missing values with the most frequent value and then we applied sklearn.preprocessing.LabelEncoder to all the categorical features.

\section{Statistical Results} 
\label{app:eval}

\begin{table}[ht]
\begin{center}
\resizebox{0.95\columnwidth}{!}{
\begin{tabular}{ l|cccccc }
 \hline
  & $SH_0$ & $SH_0$.W & $SH_1$ & $SH_0$.W & $SH_2$ & $SH_2$.W \\
 \hline
  $SH_0$    & NA & & & & & \\
  $SH_0$.W  & 3.9e-8 & NA & & & & \\
  $SH_1$    & 3.9e-8 & 1.2e-6 & NA & & & \\
  $SH_1$.W  & 3.9e-8 & 1.8e-6 & 1.3e-7 & NA & & \\
  $SH_2$    & 3.9e-8 & 2.5e-2 & 3.9e-8 & 3.3e-4 & NA & \\
  $SH_2$.W  & 3.9e-8 & 6.4-7 & 1.8e-7 & 1.2e-6  &1.3e-6 & NA \\
 \hline
\end{tabular}}
\caption{Pairwise corrected p-values for the experiments in Section \ref{exploration_vs_exploitation} (Validation Loss).}
\label{tab:1x_val}
\end{center}
\end{table}

\begin{table}[ht]
\begin{center}
\resizebox{0.95\columnwidth}{!}{
\begin{tabular}{ l|cccccc }
 \hline
  & $SH_0$ & $SH_0$.W & $SH_1$ & $SH_0$.W & $SH_2$ & $SH_2$.W \\
 \hline
  $SH_0$    & NA & & & & & \\
  $SH_0$.W  & 4.1e-9 & NA & & & & \\
  $SH_1$    & 8.6e-7 & 1.1e-7 & NA & & & \\
  $SH_1$.W  & 4.5e-9 & 8.7e-4 & 2.3e-8 & NA & & \\
  $SH_2$    & 3.5e-8 & 8.5e-3 & 5.6e-8 & 1.5e-4 & NA & \\
  $SH_2$.W  & 2.3e-8 & 5.2e-6 & 4.3e-8 & 1.4e-4 & 2.4e-6 & NA \\
 \hline
\end{tabular}}
\caption{Pairwise corrected p-values for the experiments in Section \ref{exploration_vs_exploitation} (Test Loss).}
\label{tab:1x_test}
\end{center}
\end{table}

\begin{table}[ht]
\begin{center}
\resizebox{0.95\columnwidth}{!}{
\begin{tabular}{ l|cccccc }
 \hline
  & RS & RS.W & HB & HB.W & SH & SH.W \\
 \hline
  RS    & NA & & & & & \\
  RS.W  & 8.3e-7 & NA & & & & \\
  HB    & 9.2e-5 & 2.4e-5 & NA & & & \\
  HB.W  & 1.3e-7 & 1.1e-7 & 8.3e-7 & NA & & \\
  SH    & 6.7e-6 & 6.4e-4 & 3.7e-4 & 7.3e-6 & NA & \\
  SH.W  & 9.9e-7 & 7.3e-6 & 7.3e-6 & \textbf{3.4e-1} & 7.3e-6 & NA \\
 \hline
\end{tabular}}
\caption{Pairwise corrected p-values for the experiments in Section \ref{hyperband_evaluation} (Validation Loss). Boldface indicates values larger than $0.05$.}
\label{tab:3x_val}
\end{center}
\end{table}

\begin{table}[ht]
\begin{center}
\resizebox{0.95\columnwidth}{!}{
\begin{tabular}{ l|cccccc }
 \hline
  & RS & RS.W & HB & HB.W & SH & SH.W \\
 \hline
  RS   & NA & & & & & \\
  RS.W  & 3.0e-6 & NA & & & & \\
  HB    & \textbf{9.8e-2} & 3.0e-6 & NA & & & \\
  HB.W  & 7.1e-7 & 1.1e-6 & 7.1e-7 & NA & & \\
  SH    & 8.2e-4 & 2.1e-5 & 3.9e-3 & 7.1e-7 & NA & \\
  SH.W  & 2.2e-6 & 4.1e-4 & 9.0e-7 & \textbf{1.9e-1} & 7.1e-7 & NA \\
 \hline
\end{tabular}}
\caption{Pairwise corrected p-values for the experiments in Section \ref{hyperband_evaluation} (Test Loss). Boldface indicates values larger than $0.05$.}
\label{tab:3x_test}
\end{center}
\end{table}

\bibliography{AAAI-SarigiannisD.7690}
\bibliographystyle{aaai}
\end{document}